\pdfoutput=1

\documentclass[11pt]{article}

\usepackage{acl}

\usepackage{times}
\usepackage{latexsym}

\usepackage[T1]{fontenc}

\usepackage[utf8]{inputenc}

\usepackage{microtype}

\usepackage{multirow}
\usepackage{amsfonts, amsmath, amssymb,mathtools,bbm,bm}
\usepackage[frozencache,cachedir=minted-cache]{minted}
\usepackage{listings}

\newcommand{\norm}[1]{\|#1\|}
\newcommand{\one}{\mathbbm{1}}

\usepackage{amsthm}
\theoremstyle{plain}

\newtheorem{proposition}{Proposition}

\usepackage[inline]{enumitem}

\title{Advancing Regular Language Reasoning in \\Linear Recurrent Neural Networks}

\author{
  Ting-Han Fan$^*$ \\
  Independent Researcher \\
  \texttt{tinghanf@alumni.princeton.edu} \\\AND
  Ta-Chung Chi$^*$ \\
  Carnegie Mellon University \\
  \texttt{tachungc@andrew.cmu.edu} \\\And
  Alexander I. Rudnicky \\
  Carnegie Mellon University \\
  \texttt{air@cs.cmu.edu}
}

\begin{document}

\maketitle

\begin{NoHyper}
\def\thefootnote{$^*$}\footnotetext{Equal contribution}
\end{NoHyper}
\begin{abstract}
In recent studies, linear recurrent neural networks (LRNNs) have achieved Transformer-level performance in natural language and long-range modeling, while offering rapid parallel training and constant inference cost. With the resurgence of interest in LRNNs, we study whether they can learn the hidden rules in training sequences, such as the grammatical structures of regular language. We theoretically analyze some existing LRNNs and discover their limitations in modeling regular language. Motivated by this analysis, we propose a new LRNN equipped with a block-diagonal and input-dependent transition matrix. Experiments suggest that the proposed model is the only LRNN capable of performing length extrapolation on regular language tasks such as Sum, Even Pair, and Modular Arithmetic. The code is released at \url{https://github.com/tinghanf/RegluarLRNN}.
\end{abstract}


\section{Introduction}
There is a recent surge in the use of LRNNs~\cite{gu2022s4,peng2023rwkv,Antonio2023resurrect} as alternatives to the de-facto Transformer architecture~\cite{vaswani2017attention,radford2019language}, which is ingrained in the field of natural language processing. LRNNs depart from the inter-timestep non-linearity design principle of classic RNNs~\cite{elman1990finding,jordan1997serial,hochreiter1997long,cho-etal-2014-learning}, while at the same time:
\begin{enumerate*}
    \item achieving Transformer-level performance on the task of natural language modeling~\cite{fu2023hungry,poli2023hyena} and even better performance on synthetic long-range modeling tasks~\cite{gu2022s4,gupta2022s4d,Antonio2023resurrect,hasani2023liquids4,smith2023s5}.
    \item having the added benefits of fast parallelizable training~\cite{martin2018parallelizing} and constant inference cost.
\end{enumerate*}

In spite of the remarkable empirical performance on natural language tasks, there has been no research on LRNNs' ability to model regular language. Regular language is a type of language that strictly follows certain rules like grammar.\footnote{Formally speaking, the rules are defined/recognized by the underlying finite-state machine.} The successful modeling of a regular language is important since it implies a model's ability to learn the underlying rules of the data. For example, if the training data are arithmetic operations such as $1+2\times 3$, a model should learn the rules of $a+b$, $a\times b$, and that $\times$ has a higher priority than $+$. Learning unambiguous rules behind the data is a critical step toward sequence modeling with regulated output.

In this paper, we aim to determine if existing LRNNs are competent to learn the correct grammar of regular language by testing their language transduction capability under the length extrapolation setting. Concretely, a model is trained only to predict the desired outputs on a set of short sequences of length $L_{tr}$. It then needs to predict the correct outputs for longer testing sequences of length $L_{ex}\gg L_{tr}$. Adopting the length extrapolation setting is essential to mitigate the risk of a model learning spurious shortcut solutions~\cite{liu2023transformers}.

We theoretically show that some of the recently proposed LRNNs lack the expressiveness to encode certain arithmetic operations used in the tasks of regular language. In light of this observation, we propose a new LRNN equipped with a block-diagonal and input-dependent transition matrix, which enable the successful modeling of regular language. Experiments show that the proposed model is the only LRNN architecture that can extrapolate well on regular language tasks such as Sum, Even Pair, and Modular Arithmetic.

LRNNs in this work have the following general formulation:
\begin{equation}
\begin{split}
    &x_k = A_k x_{k-1} + Bu_k\\
    &y_k = h(x_k).
    \label{eq:lrnn-general}
\end{split}
\end{equation}
$A_k$ is a matrix that defines the recurrence relation. $A_k$ may or may not depend on the input $u_k$. When it is input-independent, $A_k$ is reduced to $A$; otherwise, $A_k=g(u_k)$ for some function $g$. The first line encodes a linear recurrence in the state $x_k$. The second line is an output $y_k$ that depends on $x_k$. To control the expressiveness, the function $h$ may or may not be a linear operation.
Since the existing LRNNs differ in their linear recurrence relations (Eq.~\eqref{eq:lrnn},~\eqref{eq:dlrnn}, and \eqref{eq:liquid-recur}), we mainly focus on analyzing these relations.

\section{Limitations of Most LRNNs}
\label{sec:prior-limitation}
In this section, we theoretically show that most LRNNs are unable to represent arithmetic operations. The analysis serves as a motivation to study input-dependent transition matrices with constraints on their column norm.

\subsection{Input-independent LRNN}
\label{sec:in-indep}
To begin with, state-space models (in discrete-time format) follow the standard LRNN recurrence relation:
\begin{equation}
    x_k = A x_{k-1} + B u_k
    \label{eq:lrnn}
\end{equation}
Eq.~(\ref{eq:lrnn}) encapsulates the recurrence relation of S4~\cite{gu2022s4, gupta2022s4d}, S5~\cite{smith2023s5}, and Linear Recurrent Unit~\cite{Antonio2023resurrect}. For example, $A$ represents the HiPPO matrix family \cite{gu2023hippo} of S4 or a complex diagonal matrix of Linear Recurrent Unit. We show in Proposition~\ref{prop:input-indep} that such an input-independent matrix $A$ cannot represent subtraction.
\begin{proposition}
    An input-independent LRNN is inconsistent in representing subtraction.
    \label{prop:input-indep}
\end{proposition}
\begin{proof}
    Denote $u_0$, $u_-$, and $u_1$ as the input vector w.r.t. input characters 0, -, and 1. Denote $z$ as the initial state vector. The sequences "0-1" and "1-0" are represented as
    \begin{align*}
        &x_{0-1} = A^3z + A^2u_0 + Au_- + u_1,~~~\text{for~"0-1"}\\
        &x_{1-0} = A^3z + A^2u_1 + Au_- + u_0,~~~\text{for~"1-0"}
    \end{align*}
    Because $0-1\neq 1-0$, by forcing $x_{0-1}\neq x_{1-0}$, we have
    $$A^2u_0 + Au_- + u_1\neq A^2u_1 + Au_- + u_0.$$
    On the other hand, let $x_{0-}=A^2z+Au_0 + u_-$ be the vector representation for "0-". The sequences "0-0-1" and "0-1-0" are represented as
    \begin{align*}
        &x_{0-0-1}=A^3 x_{0-} + A^2u_0 + A u_- + u_1\\
        &x_{0-1-0}=A^3 x_{0-} + A^2u_1 + A u_- + u_0.
    \end{align*}
    Notice $x_{0-0-1}$ is for "0-0-1" while $x_{0-1-0}$ for "0-1-0". Enforcing $x_{0-0-1}= x_{0-1-0}$, we have
    $$A^2u_0 + Au_- + u_1= A^2u_1 + Au_- + u_0,$$
    which is a contradiction.
\end{proof}
The limitation described by Proposition~\ref{prop:input-indep} also applies to models adopting diagonal linear recurrence relations~\cite{gupta2022s4d,smith2023s5,Antonio2023resurrect}. The failure to represent regular language will be corroborated by the inferior length extrapolation performance reported later in \S~\ref{sec:experi}.

\section{Proposed Method}
Now that input-independent LRNNs struggle with representing arithmetic operations, we review the paradigms known to model regular language, which is the type of formal language recognized by a Finite State Automata (FSA)~\cite{chomsky1956three}. An FSA is described by a 5-tuple $(Q,\Sigma,\delta,q_0,F)$. $Q$ and $\Sigma$ are non-empty sets of states and input symbols. $q_0\in Q$ is an initial state. $\delta:Q\times\Sigma\to Q$ is an input-dependent transition function; $F\subseteq Q$ is a set of final states.

We hypothesize that an LRNN could model regular language if it can simulate an FSA, whose transition function has the following two key properties:
\begin{itemize}
    \item It is input-dependent.
    \item If represented in the matrix form, its column vectors all have unit norm (in $\|\cdot\|_1$).
\end{itemize}

\subsection{Diagonal Input-dependent LRNN}
\label{sec:dlrnn}
Let us first examine the simplest input-dependent LRNN:
\begin{equation}
    x_k = \text{diag}(v_k) x_{k-1} + B u_k,
    \label{eq:dlrnn}
\end{equation}
where $v_k=f(u_k)$ is a vector that depends on $u_k$. Unfortunately, we show that a diagonal input-dependent LRNN still cannot represent subtraction in Proposition~\ref{prop:dlrnn}.

\begin{proposition}
A diagonal input-dependent LRNN is inconsistent in representing subtraction.
\label{prop:dlrnn}
\end{proposition}
\noindent The proof is essentially a generalization of Proposition~\ref{prop:input-indep} and is deferred to Appendix~\ref{sec:appendix-dlrnn}.

\subsection{Improved Expressiveness: Liquid-S4}
\label{sec:liquid-motivation}
To improve the expressiveness of Eq.~(\ref{eq:dlrnn}), we note that the recently proposed liquid-S4~\cite{hasani2023liquids4} model has the following recurrence relation:
\begin{equation}
\begin{split}
    x_k&=Ax_{k-1} + (B u_k)\odot x_{k-1} + B u_k\\
    &=(A+\text{diag}(B u_k))x_{k-1}+B u_k,
    \label{eq:liquid-recur}
\end{split}
\end{equation}
where $\odot$ denotes the Hadamard product and $\text{diag}(w)$ constructs a diagonal matrix from $w$. Although Liquid-S4 does not suffer from the limitation outlined in Proposition~\ref{prop:dlrnn}, our experiments in \S~\ref{sec:experi_results} show that Liquid-S4 still cannot extrapolate on regular language tasks.

\subsection{Block-diagonal Input-dependent LRNN}
\label{sec:proposed}
Finally, we decide to push the expressiveness of $A_k$ to the limit and make it fully input-dependent:
\begin{equation}
    x_k = A_k x_{k-1} + B u_k,
    \label{eq:block-diag}
\end{equation}
where $A_k = g(u_k)$ is a block diagonal matrix in practice for the sake of efficiency. $A_k$ depends on $u_k$ but not previous timesteps. $g$ is an arbitrary function with the output being the size of $A_k$.

Eq.~\eqref{eq:block-diag} is numerically unstable because the product $\prod_{i=1}^k A_i$ could produce large numbers. The solution is to impose additional constraints on the norm of $A_k$:

\begin{equation}
\begin{split}
    &A_k = \text{diag}\left(A_k^{(1)},...,A_k^{(h)} \right) \in\mathbb{R}^{bh\times bh}\\
    &A_k^{(i)}= \begin{bmatrix}v_k^{(i,1)} & \dots & v_k^{(i,b)}\end{bmatrix} \in \mathbb{R}^{b\times b}\\
    &\norm{v_k^{(i,j)}}_p\leq 1,~~~i\in[1,...,h],~~~j\in[1,...,b], 
    \label{eq:pnorm-mat}
\end{split}
\end{equation}
where $\norm{\cdot}_p$ denotes the vector p-norm and $v_k^{(i,j)}$ is a column vector that depends on $u_k$. For any vector $v$, we can derive another vector $v'$ to satisfy the p-norm constraint through $v'=v / \max(1, \norm{v}_p)$. Because $\norm{v}_p\geq \norm{v}_q$ when $p\leq q$, a smaller $p$ imposes a stronger constraint on the columns of $A_k^{(i)}$. In other words, we can stabilize Eq.~\eqref{eq:block-diag} by selecting a sufficiently small $p$.

Take $p=1$ as an example. Every block $A_k^{(i)}$ is a matrix that none of its column norm is greater than 1 in $\norm{\cdot}_1$. This implies $A_{k+1}^{(i)}A_k^{(i)}$ is the same kind of matrix. Specifically, let $v^{(1)},...,v^{(b)}$ be the columns of $A_{k+1}^{(i)}A_k^{(i)}$. We have 
\begin{equation}
\begin{split}
    &\begin{bmatrix}\norm{v^{(1)}}_1 &\dots&\norm{v^{(b)}}_1\end{bmatrix}=\one^\top \left|A_{k+1}^{(i)}A_k^{(i)}\right|\\
    &\leq \one^\top \left|A_{k+1}^{(i)}\right|\left|A_k^{(i)}\right|
    \leq  \one^\top \left|A_k^{(i)}\right| \leq \one^\top.
    \label{eq:norm-1}
\end{split}
\end{equation}
Note that $\one$ is a column vector of all ones. $|\cdot|$ and $\leq$ are element-wise absolute value and inequality operations. The last two inequalities holds since the column norm of $A_{k+1}^{(i)}$ and $A_k^{(i)}$'s are no greater than 1 in $\norm{\cdot}_1$.

Eq.~\eqref{eq:norm-1} demonstrates that $p=1$ can stabilize the proposed block-diagonal recurrence, Eq.~\eqref{eq:block-diag}. However, a small $p$ restricts a model's expressiveness. In~\S~\ref{sec:experi_results}, we will show that $p=1.2$ is small enough to yield good empirical performance.

\subsection{Efficient Implementation via Parallel Scan}
\label{sec:pscan_algo}
\begin{figure*}
    \centering
    \includegraphics[width=\textwidth]{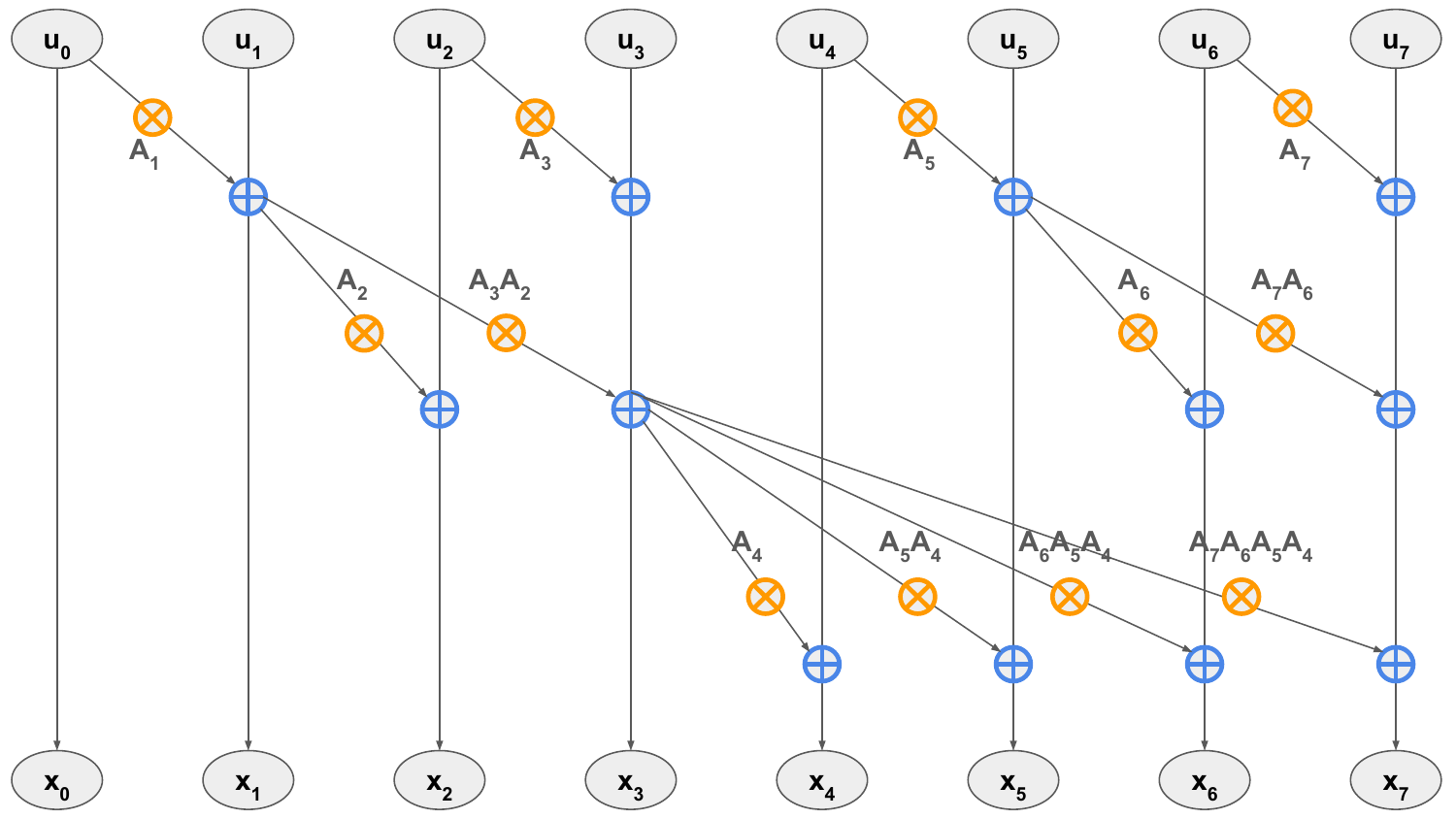}
    \caption{\textbf{Illustration of Parallel Scan for a length-8 generation.}}
    \label{fig:pscan}
\end{figure*}

We implement LRNNs in the parallel scan (\texttt{PScan}) mode as shown in Fig.~\ref{fig:pscan}. The idea of \texttt{PScan} is to group~\emph{similar} operations together, run them in parallel, and deliver the same results as those in the sequential (\texttt{Sequential}) for loop mode. For example, to compute $x_3=A_3A_2A_1u_0 + A_3A_2u_1+A_3u_2+u_3$, \texttt{Sequential} runs this in three steps. On the other hand, \texttt{PScan} decomposes the computation into two steps:
\begin{itemize}[leftmargin=*]
    \item Step 1: Compute $A_1u_0 + u_1$ and $A_3u_2 + u_3$. Because these two operations are similar, we can compute them in parallel.
    \item Step 2: $x_3=A_3A_2(A_1u_0 + u_1) + (A_3u_2 + u_3)$.
\end{itemize}

Generally speaking, a length-$L$ generation takes $\lceil\log_2 L\rceil$ steps using \texttt{PScan}. However, each step requires careful handling of the intermediate matrices. As illustrated in Fig.~\ref{fig:pscan}, for a length-8 generation, the first step requires $[A_1,A_3,A_5,A_7]$, the second step requires $[A_2,A_3A_2,A_6,A_7A_6]$, and the third step requires $[A_4,A_5A_4,A_6A_5A_4,A_7A_6A_5A_4]$. To this end,  we present an algorithm to generate the intermediate matrices in Appendix~\ref{appendix:illustration}. We integrate these intermediate matrices in \texttt{PScan} and show that \texttt{PScan} is equivalent to \texttt{Sequential} in Appendix~\ref{appendix:seq_vs_parallel}.

The computational complexity of our model is $O(b^3h\log(T))$, where $b$, $h$, and $T$ represent the block size, number of blocks, and sequence length, respectively. With the embedding dimension held fixed as $bh$, the complexity scales quadratically w.r.t the block size.

\section{Experiments}
\label{sec:experi}
\subsection{Regular Language Tasks}
We evaulate the models using the regular language transduction tasks introduced in~\citet{deletang2023Chomsky}. We prioritize language transduction over language recognition as the former can be more useful in practice~\citet{deletang2023Chomsky}. We are particularly interested in \textbf{Sum(5)}, \textbf{EvenPair(5)}, and \textbf{ModArith(5)}.

\paragraph{\textbf{Sum(M)}} The input is a string $\{s_i\}_{i=0}^{n-1}$ of numbers in $[0,..., M-1]$. The output is their sum modulo M: $\sum_{i=0}^{n-1} s_i~\text{mod}~M$. For example, when $M=5$, the input \texttt{0324} corresponds to the output \texttt{4} because $0+3+2+4~\text{mod}~5=4$. Notably, \textbf{Sum(2)} is the famous PARITY problem that evaluates whether there is an odd number of \texttt{1}s in a bit string. Thus, \textbf{Sum(M)} is a generalization of PARITY and shares the same characteristic: If one error occurs during the summation, the output will be wrong.

\paragraph{\textbf{EvenPair(M)}} The input is a string $\{s_i\}_{i=0}^{n-1}$ of numbers in $[0,..., M-1]$. The output is 1 if $s_{n-1}=s_0$ and 0 otherwise. For example, when $M=5$, the input \texttt{0320} corresponds to the output \texttt{1} because the first entry equals the last entry. Since \textbf{EvenPair(M)} only cares about the first and last entries, a model should learn to remember the first entry and forget the remaining ones $i\in[1,..,n-2]$.

\paragraph{\textbf{ModArith(M)}} The input is a string $\{s_i\}_{i=0}^{n-1}$ of odd length (i.e., $n$ is odd). The even entries ($i\in[0,2,...]$) are numbers in $[0,..., M-1]$; The odd entries ($i\in[1,3,...]$) are symbols in $\{+,-,\times\}$. The output is the answer of a mathematical expression under modulo M. For example, when $M=5$, the input \texttt{1+2-3$\times$4} corresponds to the output \texttt{1} because $1+2-3\times 4~\text{mod}~5=-9~\text{mod}~5=1.$ \textbf{ModArith(M)} is much more complicated than \textbf{Sum(M)} and \textbf{EvenPair(M)} because a model should learn to prioritize multiplication over addition and subtraction.

\subsection{Length Extrapolation}
In our pilot experiments, we discovered that all models can achieve near-perfect same-length testing accuracy; i.e., testing with $L_\text{ex}=L_\text{tr}$. This is not impossible since a large enough model can memorize all training sequences in its parameters.
To evaluate whether a model truly learns the underlying rules of a language, we first train a model on sequences of length $L_{\text{tr}}$ generated by an FSA; It is then evaluated on sequences of length $L_{\text{ex}} > L_{\text{tr}}$ generated by the same FSA.

Table~\ref{tab:extrp_settings} summarizes the extrapolation setting. We mostly follow the requirements in \citet{deletang2023Chomsky}, where the training and extrapolation lengths are 40 and 500. The lengths for \textbf{ModArith(5)} are 39 and 499 because this task requires odd-length inputs.

\begin{table}[!ht]
    \setlength{\tabcolsep}{4pt}
    \centering
    \begin{tabular}{lcccc}
    \hline\hline
    & \textbf{Sum(5)} & \textbf{EvenPair(5)} & \textbf{ModArith(5)}\\
    \hline
    $L_{\text{tr}}$& 40 & 40 & 39\\
    $L_{\text{ex}}$& 500 & 500 & 499\\
    \hline\hline 
    \end{tabular}
    \caption{\textbf{Training and Extrapolation Settings.} $L_{tr}$ and $L_{ex}$ represent the training and extrapolation sequence lengths, respectively.}
    \label{tab:extrp_settings}
\end{table}

\subsection{Baseline Models}
We select baseline LRNNs such as S4 \cite{gu2022s4}, S4D \cite{gupta2022s4d}, and Liquid-S4 \cite{hasani2023liquids4} using the released codebase\footnote{https://github.com/HazyResearch/state-spaces} under Apache-2.0 license.
These models are chosen since they are the most stable and theoretically grounded LRNN design thanks to the careful parameterization of their state transition matrices. We also experiment with RWKV \cite{peng2023rwkv} and a vanilla LRNN without S4's parameterization. Unfortunately, their performance lags behind S4 on the reported tasks.

\subsection{Experimental Results}
\label{sec:experi_results}
For the proposed method, we set $p=1.2$ in Eq.~\eqref{eq:pnorm-mat} and train the block-diagonal input-dependent LRNN with (b, h) = (8, 8). Because \textbf{ModArith} is more complicated than \textbf{Sum} and \textbf{EvenPair}, \textbf{ModArith} uses 3 layers while the others take 1 layer. Each layer is a full pass of LRNN as described in Eq.~\eqref{eq:lrnn-general}.

Table~\ref{tab:extrp_results} compares the length extrapolation capability of our model with other LRNN baselines on regular language tasks. As we can see, the proposed model is the only LRNN that can extrapolate well on regular language.
The inferior performance of S4 and S4D is expected since they cannot represent subtraction as illustrated in Prop.~\ref{prop:input-indep}. As for Liquid-S4, despite the usage of input-dependent block matrices (discussed in \S~\ref{sec:liquid-motivation}), it still cannot extrapolate well on regular language. We believe this can be explained by its low expressiveness (Eq.~\eqref{eq:liquid-recur}) compared to the proposed model (Eq.~\eqref{eq:block-diag} and \eqref{eq:pnorm-mat}). Overall, we can see that the combination of input dependency and sufficient expressiveness plays an important role in terms of regular language modeling.

\begin{table}[!ht]
    \setlength{\tabcolsep}{4pt}
    \centering
    \begin{tabular}{lcccc}
    \hline\hline
    & Ours & S4 & S4D & Liquid-S4\\
    \hline
    \textbf{Sum(5)} & 1.00 & 0.27 & 0.27 & 0.27\\
    \textbf{EvenPair(5)} & 0.99 & 0.81 & 0.82 & 0.72\\
    \textbf{ModArith(5)} & 1.00 & 0.27 & 0.27 & 0.27\\
    \hline\hline 
    \end{tabular}
    \caption{\textbf{Length Extrapolation Performance on Regular Language Tasks.} Each reported number is an average of five random trials. Each random trial returns the best testing accuracy over 40,000 gradient updates.}
    \label{tab:extrp_results}
\end{table}

\subsection{Speed Comparison}
\label{sec:speed}
We conduct our experiments using a Quadro RTX 8000 GPU. 
To provide context for the aforementioned complexity analysis in \S~\ref{sec:pscan_algo}, we take the \textbf{Sum(5)} task and set $T=40$ during the training stage. \texttt{Sequential} requires 0.033s per instance, while \texttt{PScan} completes the task in 0.021s. During the testing stage, we set $T=500$, where both \texttt{Sequential} and \texttt{PScan} take 0.03s per instance. One might anticipate \texttt{PScan} to outperform \texttt{Sequential} during testing. However, in practice, this is not the case, as the complexity incurred by $b^3$ counteracts the speedup offered by $\log(T)$. To validate our hypothesis, we set $b=1$ and reassess the speed. Subsequently, \texttt{PScan} achieves 0.0008s per instance, whereas \texttt{Sequential} takes 0.002s.
Regarding why \texttt{PScan} demonstrates a notable speedup during the training stage, we hypothesize that it is due to the improved backpropagation path enabled by \texttt{PScan}.




\section{Conclusion}
In this work, we explored LRNNs in the realm of regular language modeling. We discovered that existing LRNNs cannot effectively represent subtraction. Consequently, we proposed a new LRNN equipped with a block-diagonal and input-dependent transition matrix. Our experiments confirmed the proposed model's capability to model various regular language tasks, including Sum, Even Pair, and Modular Arithmetic, under the challenging length extrapolation setting.

\section*{Limitations}
The limitations of this work stem from several factors: (a) our evaluation is confined to only three regular language tasks; (b) the scope of our work excludes natural language; and (c) the proposed model introduces new hyperparameters such as the block size and the p-norm.

For (a), it is possible to discuss the average performance over randomly generated regular language, as demonstrated in \citet{valvoda2022rndLanguage}. Regarding (b), while natural language falls beyond the scope of our study, we believe the proposed model is at least as effective as prior linear RNN models on natural language, owing to its enhanced expressiveness. Concerning (c), the block size typically increases with the complexity of the problem. Nonetheless, it is feasible to maintain the same block size if more layers are employed (e.g., as described in \S~\ref{sec:experi_results}). Additionally, the p-norm parameter is chosen to be close to 1 to ensure stability; longer sequences correspond to smaller values of $p$.

\section*{Ethics Statement}
Our work lays the groundwork for developing LRNNs in underexplored languages, such as regular language. Inappropriate usage of our technique might have negative societal impacts, including potential losses due to wrong predictions and ethical challenges regarding the improper use of the model. These implications apply to most language processing research and are not unique to this specific work.


\bibliography{custom}
\bibliographystyle{acl_natbib}

\appendix

\section{Additional Proofs}
\label{sec:appendix}

\subsection{Proof of Proposition 2}
\label{sec:appendix-dlrnn}
Denote $(A_0,u_0)$, $(A_-,u_-)$, and $(A_1,u_1)$ as the pairs of (transition matrix, input vector) w.r.t. input characters $0$, $-$, and $1$. Note that $A_0$, $A_-$, and $A_1$ are diagonal matrices by assumption.
    
Denote $z$ as the initial state vector. The sequences \texttt{0-1} and \texttt{1-0} are represented as
\begin{align*}
    &x_{0-1} = A_1A_-A_0z + A_1A_-u_0 + A_1u_- + u_1\\
    &x_{1-0} = A_0A_-A_1z + A_0A_-u_1 + A_0u_- + u_0.
\end{align*}
Note that $x_{0-1}$ is \texttt{0-1} and $x_{1-0}$ is \texttt{1-0}.
Because the $A$ matrices are diagonal, we know $A_1A_-A_0=A_0A_-A_1$. Because $0-1\neq 1-0$, by enforcing $x_{0-1}\neq x_{1-0}$, we have
\begin{equation}
    A_1A_-u_0 + A_1u_- + u_1 \neq A_0A_-u_1 + A_0u_- + u_0.
    \label{eq:0-1vs1-0}
\end{equation}
On the other hand, let $x_{0-}=A_-A_0z+A_-u_0 + u_-$ be the vector representation for "0-". Consider two other sequences \texttt{0-0-1} and \texttt{0-1-0}, their vector representations are
\begin{align*}
&x_{0-0-1}=A_1A_-A_0 x_{0-} + A_1A_-u_0 + A_1 u_- + u_1\\
&x_{0-1-0}=A_0A_-A_1 x_{0-} + A_0A_-u_1 + A_0 u_- + u_0.
\end{align*}
Note $x_{0-0-1}$ is \texttt{0-0-1} and $x_{0-1-0}$ is \texttt{0-1-0}. Similarly, because the $A$ matrices are diagonal and $0-0-1= 0-1-0$, by enforcing $x_{0-0-1}=x_{0-1-0}$, we have
\begin{equation}
    A_1A_-u_0 + A_1u_- + u_1 = A_0A_-u_1 + A_0u_- + u_0.
    \label{eq:0-0-1vs0-1-0}
\end{equation}
Because Eq.~\eqref{eq:0-1vs1-0} contradicts Eq.~\eqref{eq:0-0-1vs0-1-0}, the two relations $x_{0-1}\neq x_{1-0}$ and $x_{0-0-1}=x_{0-1-0}$ cannot co-exist. We hence conclude that an input-dependent diagonal linear RNN is inconsistent in representing subtraction.

\subsection{Code for \texttt{PScan}}
\onecolumn
\subsubsection{Illustration of Matrix Generation}
\label{appendix:illustration}
\begin{minted}[
frame=lines,
framesep=2mm,
baselinestretch=1.2,
fontsize=\scriptsize,
]{python}
import numpy as np
seq_len = 2**3 - 1
arr = np.array(['A' + str(i) for i in range(1,seq_len +1)]).reshape(-1,1)

def spt(x):
    assert len(x)%2 == 1, 'works when len(x)== 2**k -1 for k>=1'
    coef = x[::2]
    remain = x[1::2]

    coef_remain = np.core.defchararray.add(coef[1:], remain[:,-1:])
    remain = np.concatenate([remain, coef_remain], axis=1)
    return coef, remain

for i in range( int(np.ceil(np.log2(seq_len))) ):
    coef, arr = spt(arr)
    print(coef)
\end{minted}

The below output shows the function \texttt{spt()} can generate the intermediate matrices during \texttt{PScan}.
{\small
\begin{lstlisting}[language=Python]
[['A1']
 ['A3']
 ['A5']
 ['A7']]
[['A2' 'A3A2']
 ['A6' 'A7A6']]
[['A4' 'A5A4' 'A6A5A4' 'A7A6A5A4']]
\end{lstlisting}
}

\subsubsection{Testing the Equivalence of \texttt{Sequential} and \texttt{PScan}}
\label{appendix:seq_vs_parallel}
\begin{minted}[
frame=lines,
framesep=2mm,
baselinestretch=1.2,
fontsize=\scriptsize,
]{python}
import numpy as np
import torch
import torch.nn as nn
torch.manual_seed(1)
emb_dim = 2
seq_len = 7
bs = 1
A = torch.randn(bs, seq_len, emb_dim, emb_dim)
u = torch.randn(bs, seq_len, emb_dim)
x0 = torch.randn(1, emb_dim)

# sequential
x = x0.expand(bs, emb_dim)
all_x = [x[:,None,:]]
for i in range(seq_len):
    x = torch.einsum('bij,bj->bi', A[:,i], x) + u[:,i]
    all_x.append(x[:,None,:])
all_x = torch.cat(all_x, dim=1)
print('sequential mode')
print(all_x)


# parallel scan
def scan(x, As):
    c = As.shape[2]*2
    x = x.view(bs, L//c, c, -1)
    x1, x2 = x[:,:,:c//2], x[:,:,c//2:]

    # x2.shape = (bs, group nums, group size, emb_dim)
    # As.shape = (bs, group nums*2-1, group size, emb_dim, emb_dim)

    assert As.shape[1]%2==1, 'works when As.shape[1]== 2**k -1 for k>=1'
    coef = As[:,::2]
    remain = As[:,1::2]
    prodd = torch.einsum('bncij,bnjk->bncik', coef[:,1:], remain[:,:,-1])
    remain = torch.cat([remain, prodd], dim=2)

    # coef.shape = (bs, group nums, group size, emb_dim, emb_dim)
    # apply a group of matrix (e.g., ['A2' 'A3A2']) to the last element of x2 in each group,
    # and add together
    x2 = x2 + torch.einsum('bncij,bnj->bnci', coef, x1[:,:,-1])
    x = torch.cat([x1, x2], dim=2)

    return x, remain

log2_L = int(np.ceil(np.log2(seq_len+1)))
L = 2**log2_L # the length after zero padding
n_zero = L - seq_len - 1
eu = torch.cat([x0.expand(bs,-1)[:,None,:], u], dim=1)
eu = nn.functional.pad(eu, (0,0,0, n_zero))


x = eu
As = nn.functional.pad(A, (0,0,0,0,0, n_zero))[:,:,None,:,:]

for i in range(log2_L):
    x, As = scan(x, As)
x = x.view(bs, L, emb_dim)[:,:seq_len+1,:]
print('parallel mode')
print(x)
\end{minted}

The below shows that \texttt{Sequential} and \texttt{PScan} are equivalent as they generate the same outputs.

{\small
\begin{lstlisting}[language=Python]
sequential mode
tensor([[[ 0.8310, -0.2477],
         [ 0.5167, -1.4218],
         [ 1.1399,  1.3024],
         [ 0.9628,  1.3150],
         [-1.5308, -1.6903],
         [-3.6631,  1.6082],
         [ 1.7805,  7.1659],
         [ 2.5068, -0.6256]]])
parallel mode
tensor([[[ 0.8310, -0.2477],
         [ 0.5167, -1.4218],
         [ 1.1399,  1.3024],
         [ 0.9628,  1.3150],
         [-1.5308, -1.6903],
         [-3.6631,  1.6082],
         [ 1.7805,  7.1659],
         [ 2.5068, -0.6256]]])
\end{lstlisting}
}
\end{document}